\documentclass[sigconf]{aamas}

\usepackage{balance} % for balancing columns on the final page

\usepackage{bm}
\usepackage{amsmath,amsfonts}
\usepackage{paralist}
\usepackage{wrapfig}
\theoremstyle{definition}
\newtheorem{definition}{Definition}

\usepackage{algorithm}
\usepackage{algorithmic}
\usepackage{hyperref}
\usepackage{float}

\makeatletter
\gdef\@copyrightpermission{
  \begin{minipage}{0.2\columnwidth}
   \href{https://creativecommons.org/licenses/by/4.0/}{\includegraphics[width=0.90\textwidth]{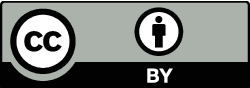}}
  \end{minipage}\hfill
  \begin{minipage}{0.8\columnwidth}
   \href{https://creativecommons.org/licenses/by/4.0/}{This work is licensed under a Creative Commons Attribution International 4.0 License.}
  \end{minipage}
  \vspace{5pt}
}
\makeatother

\setcopyright{ifaamas}
\acmConference[AAMAS '25]{Proc.\@ of the 24th International Conference
on Autonomous Agents and Multiagent Systems (AAMAS 2025)}{May 19 -- 23, 2025}
{Detroit, Michigan, USA}{Y.~Vorobeychik, S.~Das, A.~Nowé  (eds.)}
\copyrightyear{2025}
\acmYear{2025}
\acmDOI{}
\acmPrice{}
\acmISBN{}

\acmSubmissionID{<<OpenReview submission id>>}

\title[Tackling Uncertainties in MARL]{Tackling Uncertainties in Multi-Agent Reinforcement Learning through Integration of Agent Termination Dynamics}

\author{Somnath Hazra}
\affiliation{
  \institution{IIT Kharagpur}
  \city{Kharagpur}
  \country{India}}
\email{somnathhazra@kgpian.iitkgp.ac.in}

\author{Pallab Dasgupta}
\affiliation{
  \institution{Synopsys}
  \city{Santa Clara}
  \country{USA}}
\email{pallabd@synopsys.com}

\author{Soumyajit Dey}
\affiliation{
  \institution{IIT Kharagpur}
  \city{Kharagpur}
  \country{India}}
\email{soumya@cse.iitkgp.ac.in}

\begin{abstract}
Multi-Agent Reinforcement Learning (MARL) has gained significant traction for solving complex real-world tasks, but the inherent stochasticity and uncertainty in these environments pose substantial challenges to efficient and robust policy learning. While Distributional Reinforcement Learning has been successfully applied in single-agent settings to address risk and uncertainty, its application in MARL is substantially limited. In this work, we propose a novel approach that integrates distributional learning with a safety-focused loss function to improve convergence in cooperative MARL tasks. Specifically, we introduce a Barrier Function based loss that leverages safety metrics, identified from inherent faults in the system, into the policy learning process. This additional loss term helps mitigate risks and encourages safer exploration during the early stages of training. We evaluate our method in the StarCraft II micromanagement benchmark, where our approach demonstrates improved convergence and outperforms state-of-the-art baselines in terms of both safety and task completion. Our results suggest that incorporating safety considerations can significantly enhance learning performance in complex, multi-agent environments.
\end{abstract}



\newcommand{\BibTeX}{\rm B\kern-.05em{\sc i\kern-.025em b}\kern-.08em\TeX}

\begin{document}

%%% The following commands remove the headers in your paper. For final 
%%% papers, these will be inserted during the pagination process.

\pagestyle{fancy}
\fancyhead{}

%%% The next command prints the information defined in the preamble.

\maketitle 

%%%%%%%%%%%%%%%%%%%%%%%%%%%%%%%%%%%%%%%%%%%%%%%%%%%%%%%%%%%%%%%%%%%%%%%%

\section{Introduction}
\label{sec1}

Multi-Agent Systems (MAS) remain a highly challenging domain, owing to the inherent stochasticity in the environment, and concurrently learning agents; resulting in uncertain reward outcomes. Therefore the contemporary Multi-Agent Reinforcement Learning (MARL) literature has shifted from point estimates derived from expected global rewards \cite{sunehag2017value, rashid2020monotonic} to distributional analysis over returns \cite{bellemare2017distributional, dabney2018implicit} which provide a more informed understanding of the variability in future returns. However, the predicted distributions are prone to errors at the initial phases of training because the policy has not yet explored and trained enough to predict accurate distributions. In the context of MARL, the prediction errors become even more critical due to the collective uncertainty from multiple agents, where each agent's decisions can drastically impact both individual and collective outcomes. Our focus lies specifically on cooperative MARL, where a team of agents collaborates to achieve a shared objective in a stochastic environment \citep{stefano2024multi}.

Despite the effectiveness of the Centralized Training and Decentralized Execution (CTDE) paradigm \citep{lowe2017multi} for cooperative MARL most CTDE-based algorithms maximize returns and overlook the uncertainty due to untrained parameters in the policy that captures the inherent stochasticity in MAS. Here agents must contend with not only uncertain rewards but also additional safety constraints resulting from the stochasticity. Our contribution is to identify the innate fault tolerant MARL formulation from the system such that the system does not go beyond a pre-defined safety boundary. For example, a team game cannot be won if players keep getting eliminated. These safety objectives hold importance during the early phases of training since it is hard to estimate accurate distributions by a distributional model owing to exploration strategies and partial observability of the agents.

\begin{figure}[!ht]
    \centering
    \includegraphics[width=0.8\linewidth]{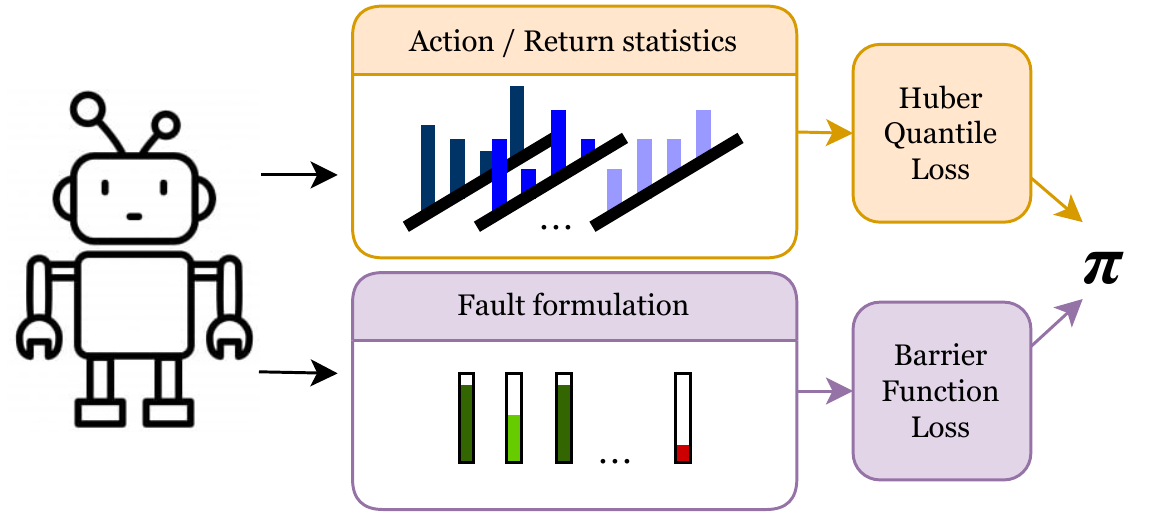}
    \caption{Integration of the safety constraints along with loss functions calculated with respect to returns, helps reduce the uncertainity, especially during the initial phase of training.}
    \label{fig1}
\end{figure}

As an instance of such a situation, let us consider a battle scenario between two agent groups. The reward can be based on intermediate strategies decided by the teams. But it is unsure whether their strategy will work, since any agent army does not know about the strategies planned by the other army. Moreover the agents may not be experts in carrying out their decided objective. However, one property that is overlooked, is that the battle cannot be won if there are not enough agents left to fight; i.e., there is an inherent fault tolerant formulation in the system. For example, in battle simulations like StarCraft II, minimizing ally casualties can be as important as maximizing reward \citep{samvelyan2019starcraft}.

Motivated by the gap in current research, our work addresses the uncertainties of the system via integration of safety considerations from the inherent faults in the MAS. While previous efforts have successfully combined Distributional RL with MARL to improve risk sensitivity \citep{qiu2020rmix, sun2021dfac}, these approaches largely focus on algorithmic modifications rather than optimizing the learning process by using the available information regarding the safety-critical objectives. Our contributions are summarized as follows:
\begin{itemize}
    \item We propose a novel loss function that incorporates safety constraints into the optimization process, derived from the safety boundary of the inherent faults in the system. By explicitly accounting for safety alongside reward maximization, we aim to accelerate learning, tackle environmental uncertainties, and improve long-term policy performance in risk-sensitive environments. We also use the Dueling Networks \citep{wang2016dueling} in our architecture to eliminate ineffective actions.
    \item We also suggest an architectural improvement for the local policy network in the CTDE paradigm to utilize the return distribution for prioritizing the important components from the local observation.
    \item We validated our approach on the StarCraft II micromanagement \citep{samvelyan2019starcraft} and MetaDrive\cite{li2021metadrive} benchmarks, demonstrating its efficacy compared to existing baselines and state-of-the-art distributional MARL algorithms.
\end{itemize}

More specifically, our work contributes to bridging the gap between theoretical advancements in Distributional MARL and practical needs in real-world MARL environments, particularly those with critical safety concerns.

%%%%%%%%%%%%%%%%%%%%%%%%%%%%%%%%%%%%%%%%%%%%%%%%%%%%%%%%%%%%%%%%%%%%%%%%

\section{Preliminaries}
\label{sec2}

This section outlines some of the preliminary concepts.

\subsection{Decentralized POMDP}

The de facto standard for modeling cooperation in MARL is a Decentralized Partially Observable MDP or Dec-POMDP \citep{oliehoek2008optimal}. It is defined using the tuple $\mathcal{M} = \langle \mathcal{N}, \mathcal{S}, \mathcal{U}, \Omega, P, \mathcal{O}, r, \gamma \rangle$, where $\mathcal{N} \equiv \{a_1, a_2, ..., a_n\}$ is a finite set of agents, and $s \in \mathcal{S}$ represents the global state of the system. At each time-step, the individual agents choose an action $u_i \in \mathcal{U}$ for $i \in \mathcal{N}$ based on observation $o_i \in \Omega$, with probability $\mathcal{O}(o_i | s, u_i)$. The joint action, $\textbf{u} = [u_i]_{i=1}^{n}$, changes the state of the system from $s$ to $s' \sim P(\cdot | s, \textbf{u})$; yielding a joint reward $r(s, \textbf{u})$. $\gamma \in [0, 1)$ is the discount factor.

Value functions provide an estimate of the reward that can be obtained given a state, or when taking an action from a state. Accordingly, the state value function for the system is given by, $V_{\bm{\pi}}(s) = \mathbb{E}[\sum_{t=0}^{\infty} \gamma^{t} r(s_t, \textbf{u}_t) | s_0=s, \bm{\pi}]$; and the state-action value function is given by, $Q_{\bm{\pi}}(s, \textbf{u}) = r(s, \textbf{u}) + \gamma \mathbb{E}_{s'}V_{\bm{\pi}}(s')$. The formal objective of the co-operative system is to find a joint policy $\bm{\pi} \equiv \langle \pi_1, ..., \pi_n \rangle$ that maximizes the expected return given by:

\begin{equation}
\label{eqn1}
    \bm{\pi}* = \arg \max_{\bm{\pi}} \mathbb{E}_{(s_t, \textbf{u}_t) \sim (P, \bm{\pi})} \left[ \sum_{t=0}^{\infty} \gamma^{t} r(s_t, \textbf{u}_t) \right]
\end{equation}

\subsection{Centralized Training and Decentralized Execution}

Multi-agent algorithms generally follow the Centralized Training Decentralized Execution (CTDE) paradigm \cite{lowe2017multi}. Agents learn collaboratively during training using feedback from the environment. This phase may involve information exchange with a central approximation network. This phase is termed as the Central Training phase. The Central Value function accumulates the values of the individual policies $\pi_i$ in a forward propagation determined by the \emph{Credit Assignment} strategy \cite{chang2003all} of the algorithm. The strategies build on the \emph{Individual Global Max (IGM)} principle which states that locally optimal actions maximize the global reward of the system given by the following equation \cite{hostallero2019learning}.

\begin{equation}
    \arg \max_{\textbf{u} \in \mathcal{U}^n} Q_{\bm{\pi}}(s, \textbf{u}) = \left[ \arg \max_{u_i \in \mathcal{U}} Q_{\pi_i}(o_i, u_i) \right]_{i=1}^{n}
\end{equation}

During the Decentralized Execution phase, agents act independently based on local observations.

\subsection{Distributional RL}

In single-agent RL, distributional RL aims to learn the complete distribution over the returns, $Z(s, \mathbf{u})$, instead of learning the expected return, $Q(s, \mathbf{u})$. The distribution over the returns can be approximated using a categorical distribution \cite{bellemare2017distributional} or a quantile function \cite{dabney2018implicit}. We here use Implicit Quantile Network (IQN) \cite{dabney2018implicit}, that models $Z(s, \mathbf{u})$ using an inverse cumulative distribution function (CDF), $F^{-1}(s, \mathbf{u} | w)$, $w \in [0, 1]$. IQN defines the distributional Bellman operator to update $Z(s, \mathbf{u})$ using the Huber quatile regression loss, a distributional version of the temporal difference error. During execution, the action is chosen based on the largest expected return, $\arg \max_{\textbf{u}} \mathbb{E}[Z(s, \mathbf{u})]$. The distributional analysis in MARL is based on the \emph{Distributional IGM (DIGM)} principle, which states that, given individual state-action value distributions, $[Z_{\pi_i}(o_i, u_i)]_{i=1}^n$, the joint state-action value distribution, $Z_{\bm{\pi}}(s, \mathbf{u})$, is maximized if the following equation holds \cite{sun2021dfac}.

\begin{equation}
    \arg \max_{\textbf{u} \in \mathcal{U}^n} Z_{\bm{\pi}}(s, \textbf{u}) = \left[ \arg \max_{u_i \in \mathcal{U}} Z_{\pi_i}(o_i, u_i) \right]_{i=1}^{n}
\end{equation}

%%%%%%%%%%%%%%%%%%%%%%%%%%%%%%%%%%%%%%%%%%%%%%%%%%%%%%%%%%%%%%%%%%%%%%%%

\section{Related Work}
\label{sec3}

Efficiency and scalability have been central challenges in the MARL literature, particularly when estimating the global expected return (for using the reward, $r$ as a reference during training) from local policies (in decentralized execution). Value factorization methods have been widely adopted to address this, adhere to the IGM principle. Initial approaches like VDN \citep{sunehag2017value} directly sum local value functions; while QMIX \cite{rashid2020monotonic} uses a monotonic mixing function to ensure compatibility with the global return. QTRAN \cite{son2019qtran} introduced linear constraints to factorize the global return without the monotonicity assumption. Transformer-based architectures, such as Qatten \cite{yang2020qatten}, have also been explored in MARL. QPLEX \cite{wang2020qplex} leverages the dueling network architecture \cite{wang2016dueling} to improve learning generalization across actions, which we have used in a distributional setting in this work. Despite these advancements, the stochasticity inherent in multi-agent systems complicates the use of expected returns, often hindering training efficiency.

In parallel, distributional RL has made significant strides in the single-agent domain, introducing algorithms such as C51 \cite{bellemare2017distributional}, implicit quantile networks (IQN) \cite{dabney2018implicit}, and others \cite{dabney2018distributional, rowland2019statistics}. More recently, there have been increasing interest in unifying distributional RL with MARL to improve risk-sensitivity and robustness. Notable works include restructuring the IGM principle to a distributional Q-learning standpoint \citep{sun2021dfac, shen2024riskq}; or exploring the methods for aggregating reward distributions for each action, considering the sources of risk \citep{son2021disentangling, oh2022risk}. Works such as RMIX \cite{qiu2020rmix} integrate risk-aware metrics like Conditional Value at Risk (CVaR) into the QMIX framework, while DFAC \cite{sun2021dfac} introduces the Distributional IGM (DIGM) principle, extending value factorization methods like VDN and QMIX to the distributional setting. DIGM was further refined in subsequent work \cite{sun2023unified}. ResQ \cite{shen2022resq} builds upon DMIX by introducing residual Q-functions for more accurate return estimation, while other studies have explored risk-sensitive aggregation of reward distributions \cite{son2021disentangling, oh2022risk}, uncertainty-aware exploration strategies \cite{oh2023toward}, and modelling uncertainties related to reward \cite{hu2022distributional}.

Action shielding has been another widely explored venue in RL and MARL to prevent unsafe actions during training and execution. Alshiekh et al. \cite{alshiekh2018safe} introduced a shielding mechanism to enforce safety constraints by overriding unsafe actions based on a safety specification. Building on this, Bharadwaj et al. \cite{bharadwaj2019synthesis} proposed minimum-cost shields for multi-agent systems to ensure safe coordination. Factored shields curated on linear temporal logic properties \cite{elsayed2021safe} is another venue that has been explored for MARL. For environments with nonlinear dynamics, model predictive shielding \cite{bastani2021safe} provides an efficient approach to maintaining safety. Zhang et al. proposed a multi-agent version \cite{zhang2019mamps}, where only a subset of agents, determined by a greedy algorithm, need to use a backup policy. While these methods focus on action-level intervention, our approach incorporates safety directly into the policy learning process via the barrier function, complementing action shielding techniques by prioritizing long-term safety considerations during training.

Our work builds on these existing distributional MARL approaches by incorporating agent-specific information to enhance policy learning; unlike previous studies that focus primarily on algorithmic adaptations of MARL to the distributional context. This enables more effective training in stochastic multi-agent environments.

%%%%%%%%%%%%%%%%%%%%%%%%%%%%%%%%%%%%%%%%%%%%%%%%%%%%%%%%%%%%%%%%%%%%%%%%

\section{Methodology}
\label{sec4}

The dynamic nature of MARL environments arises from multiple agents acting concurrently, often following independent local heuristics (policies). This interaction leads to inherent uncertainty in the outcomes (returns), especially during the initial phases of training. Although exploration strategies can help reduce the epistemic uncertainty by utilizing the generalization capability of neural networks, it remains an integral part of multi-agent environments at the start of training \cite{mavrin2019distributional, jiang2024importance}. To address this challenge, we propose leveraging the safety implications available from the environment, e.g., agent deaths, to boost convergence during policy learning. We use a loss function based on Control Barrier Function (CBF) \citep{qin2021learning, yang2023model}, alongside the Huber quantile loss, for training the local policies. In the following sub-section we formally describe the Barrier Function loss used in this work.

\subsection{The Barrier Function Loss}

Barrier Function is employed to ensure that the agents' trajectories stay within a "safe" region of the state space. This barrier defines constraints that the policies must satisfy, ensuring that unsafe states are avoided during training. Formally, the barrier certificate is defined as follows \citep{yang2023model}.

\begin{definition}
  \textbf{Barrier Certificate}. A barrier certificate, $B^{\bm{\pi}}$, for a policy, $\bm{\pi}$ is a function $B^{\bm{\pi}}: \mathcal{S} \rightarrow \mathbb{R}$, such that the following properties hold.
  \begin{enumerate}
      \item $\forall s \in \mathcal{S}_0, B^{{\bm{\pi}}}(s) \leq 0$
      \item $\forall s \in \mathcal{S}_{\text{unsafe}}, B^{{\bm{\pi}}}(s) > 0$
      \item $\forall s \in \mathcal{S}, B^{{\bm{\pi}}}(s') - B^{{\bm{\pi}}}(s) \leq -\lambda_B B^{{\bm{\pi}}}(s)$ 
  \end{enumerate}
  where $s' \sim P(\cdot | s, \textbf{u})$, $\textbf{u} \sim \bm{\pi}$, and $0 < \lambda_B < 1$ is a hyper-parameter controlling the convergence rate.
\end{definition}

In our work, the barrier function is designed to capture the number of agent terminations as a function of the state over the course of a trajectory. Specifically, the barrier function is defined using the following equation.
\begin{equation}
    B^{{\bm{\pi}}}(s) = (\text{agents dead at } s) + \gamma_B B^{{\bm{\pi}}}(s')
\end{equation}
where $s' \sim P(\cdot | s, \textbf{u})$, and $\gamma_B$ signifies the discount factor for the barrier function. Intuitively, it can be thought of as an estimation of vulnerability (opposite of agent health) of the collective group of agents. The formulation bears similarity with the definition of the value function $V_{\bm{\pi}}$, as it includes the discounted future state estimate at $s'$ in the current estimate; since in an MDP the state $s' \sim P(\cdot | s, \cdot)$. Among the three conditions above, our policy should satisfy the third condition (which is the invariant property), i.e., $\forall s \in \mathcal{S}, B^{{\bm{\pi}}}(s') - B^{{\bm{\pi}}}(s) \leq -\lambda_B B^{{\bm{\pi}}}(s)$, ensuring that the barrier function consistently decreases along the trajectory \citep{yang2023model}. 

To enforce this condition, we integrate the following loss function for policy optimization:

\begin{equation}
    \mathcal{L}_B(\bm{\pi}) = \frac{1}{|\mathcal{S}|} \sum_{\substack{s,s' \in \mathcal{S} \\ s' \sim P(\cdot | s, \textbf{u})}} \max \left( B^{{\bm{\pi}}}(s') - (1 - \lambda_B) B^{{\bm{\pi}}}(s) , 0 \right)
\end{equation}
where $|\mathcal{S}|$ represents the number of visited states over a trajectory. This loss function is derived from the number of agent deaths, which is a safety-critical metric. By focusing on this deterministic safety criterion, the method improves the accuracy of policy updates in unsafe scenarios.

An alternative approach could be to calculate the barrier loss for each individual agent, and update the local policies $\pi_i$ accordingly. But this approach has two notable drawbacks;
\begin{inparaitem}
    \item given the length of the trajectory, the impact of this loss at each state may be a insufficient to derive meaningful gradient updates, particularly when averaged across the trajectory,
    \item in a cooperative environment, it is unrealistic to attribute the responsibility of an agent’s termination solely to the corresponding local policy, as the outcome is influenced by the actions of all agents.
\end{inparaitem}

Thus, our method leverages the global safety constraint across the trajectory, promoting stable and efficient convergence in MARL.

\begin{figure*}[!t]
    \centering
    \includegraphics[width=0.8\linewidth]{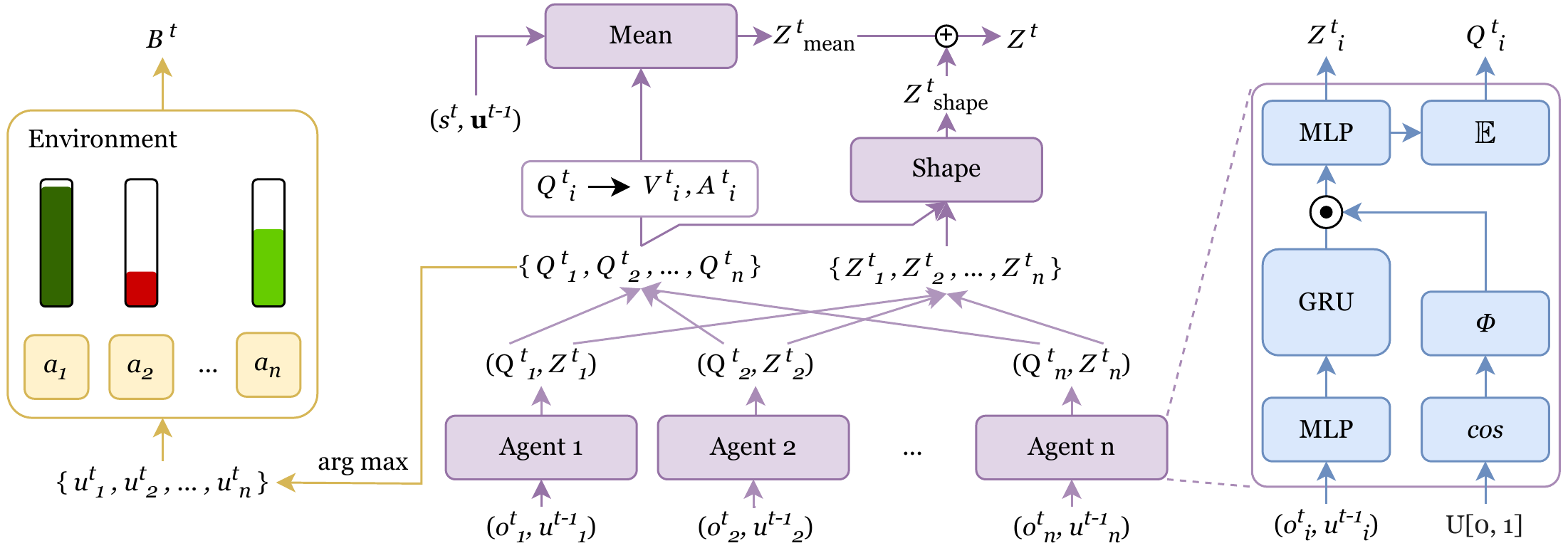}
    \caption{The overall framework illustrating the training objective. The calculated errors are combined using gradient manipulation and back-propagated through the central training model. References to the abbreviated notations from the text: $Z_i^t = Z_{\pi_i}(o_i^t, u_i^t)$, $Q_i^t = Q_{\pi_i}(o_i^t, u_i^t)$, $t$ is time, $A_i^t = A_{\pi_i}(o_i^t, u_i^t)$, $V_i^t = V_{\pi_i}(o_i^t)$.}
    \label{fig2}
\end{figure*}

\subsection{Optimization on Return}

After defining the safety objective through the CBF loss, we integrate it with the DRL objective to optimize the reward function. We use the Huber-quantile regression loss \cite{dabney2018implicit} for $Q_{\bm{\pi}}$, allowing us to model the stochastic return distribution effectively.

In cooperative multi-agent systems, value factorization techniques assist CTDE. However, to account for the stochasticity of value functions, we require a distributional variant of value factorization that satisfies Distributional IGM \cite{sun2021dfac}. The Mean-Shape decomposition \cite{sun2021dfac, sun2021distributional} addresses this by separating distribution $Z$ over global utilities into: $Z = \mathbb{E}[Z] + (Z - \mathbb{E}[Z]) = Z_{\text{mean}} + Z_{\text{shape}}$. The mean term in the mean-shape decomposition $\mathbb{E}[Z]$, represents the expected value, while the shape term captures deviations from this expectation \cite{lyle2019comparative, sun2021distributional}. This ensures compliance with DIGM while allowing effective factorization of stochastic utilities.

For factorization, we adopt a distributional version of the QPLEX algorithm \citep{wang2020qplex}, using a dueling network architecture \citep{wang2016dueling}. Following \citep{sun2023unified}, the global action-value function $Q_{\bm{\pi}}$ is factorized as follows.

\begin{align}
\begin{split}
    Q_{\bm{\pi}}(s, \mathbf{u}) = & \sum_{i=1}^n V_{\pi_i}(o_i) + \sum_{i=1}^n \lambda_i(o_i, u_i) \cdot A_{\pi_i}(o_i, u_i) \\
    & + \sum_{i=1}^n (Z_{\pi_i}(o_i, u_i) - Q_{\pi_i}(o_i, u_i))
\end{split}
\end{align}
Here $V_{\pi_i}(o_i)$ denotes the value function for agent $i$'s observation, $A_{\pi_i}(o_i, u_i)$ is the advantage function, and $Z_{\pi_i}(o_i, u_i)$ captures the distributional nature of local utilities. The overall architecture for obtaining the global distribution over returns and the CBF value from the environment is illustrated in Figure \ref{fig2}.

\subsection{Overall Training Objective}

Now that we have obtained the formulations for calculating both the loss functions, we need to integrate them to calculate our overall loss for the gradient descent. In the following formulation for calculating the overall loss, we consider the Huber quantile loss as $\mathcal{L}_Q$ and the CBF loss as $\mathcal{L}_B$. We combine the loss functions using gradient manipulation, where the gradient is calculated from the backpropagation of the individual loss function components, as studied in \cite{yu2020gradient} for multi-task learning.

The objective of gradient manipulation is to reduce the deviation of gradients due to the reward and barrier components. Let the gradient due to $\mathcal{L}_B$ be represented as $g_B$, and that due to the $\mathcal{L}_Q$ as $g_Q$. Let $\theta$ be the angle between $g_B$ and $g_Q$. Since we work with gradient manipulation here, the gradients occasionally may need to be projected to a plane that is normal to the other gradient; we will explain this in more detail in the following paragraph. Let the projection of gradient of $g_B$ on the normal plane of $g_Q$ be given by $g_B^+$; while the projection of gradient of $g_Q$ on the normal plane of $g_B$ be represented by $g_Q^+$.

\begin{figure}[!ht]
    \centering
    \includegraphics[width=0.85\linewidth]{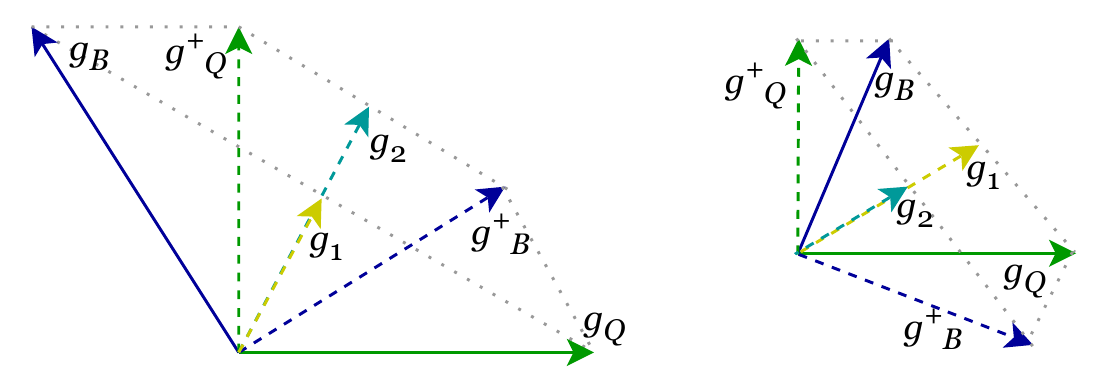}
    \caption{Analysis of gradient manipulation.}
    \label{fig4}
\end{figure}

Owing to optimization on two tasks, there may occasionally occur a conflict among the gradients $g_Q$ and $g_B$, i.e. $\theta > 90^{\circ}$. In those cases, it is essential to project the gradient vector onto the normal plane of the other vector for effective updates to the network \cite{yu2020gradient}. For cases when $\theta \leq 90^{\circ}$, i.e., the gradients do not conflict, the gradient manipulation is not required. The overall gradient calculation, $g$, is thus given by the following equation.

\begin{align}
    \label{eqn7}
    g = \begin{cases}
        \beta_Q^+ \cdot g_Q^+ + \beta_B^+ \cdot g_B^+ , \quad \theta > 90^{\circ} & = g_1 \text{(let)} \\
        \beta_Q \cdot g_Q + \beta_B \cdot g_B , \quad \theta \leq 90^{\circ} & = g_2 \text{(let)}
    \end{cases} \\
    g_Q^+ = g_Q - \frac{g_Q \cdot g_B}{|| g_B ||^2} g_B  \quad \quad
    g_B^+ = g_B - \frac{g_B \cdot g_Q}{|| g_Q ||^2} g_Q
\end{align}
where $\beta_Q^+, \beta_B^+, \beta_Q, \beta_B$ denote the weights of the gradients. With this strategy we can optimize on both the tasks more efficiently. As shown in Equation \ref{eqn7}, we have two cases for the gradients corresponding to the loss functions calculated: when the gradients conflict, i.e., $\theta > 90^{\circ}$; and when the gradients do not conflict, i.e., $\theta \leq 90^{\circ}$. The process is illustrated in Figure \ref{fig4}.

1. \emph{When $\theta > 90^{\circ}$}: We consider $\beta_Q^+$ and $\beta_B^+$ are equal to 0.5. Then according to the given equations, we have the following.
\begin{align*}
    g_1 & = 0.5 \cdot g_Q^+ + 0.5 \cdot g_B^+ \\
    & = 0.5 \cdot \left( g_Q - \frac{g_Q \cdot g_B}{|| g_B ||^2} g_B \right) + 0.5 \cdot \left( g_B - \frac{g_B \cdot g_Q}{|| g_Q ||^2} g_Q \right) \\ 
    g_2 & = 0.5 \cdot g_Q + 0.5 \cdot g_B
\end{align*}
We know that, when $\theta > 90^{\circ}$, $(g_Q \cdot g_B) < 0 \Rightarrow -(g_Q \cdot g_B) > 0$. Therefore, according to the above equations, $g_1 > g_2$. So when $\theta > 90^{\circ}$, we use $g_1$ as our gradient.

2. \emph{When $\theta \leq 90^{\circ}$}: We consider $\beta_Q$ and $\beta_B$ are equal to 0.5. We know that, when $\theta \leq 90^{\circ}$, $(g_Q \cdot g_B) \geq 0 \Rightarrow -(g_Q \cdot g_B) \leq 0$. Therefore, according to the above equations, $g_1 \leq g_2$. So when $\theta \leq 90^{\circ}$, we use $g_2$ as our gradient.

\subsection{The Local Policy Network}
\label{sec4_4}

\begin{figure}[!ht]
    \centering
    \includegraphics[width=0.65\linewidth]{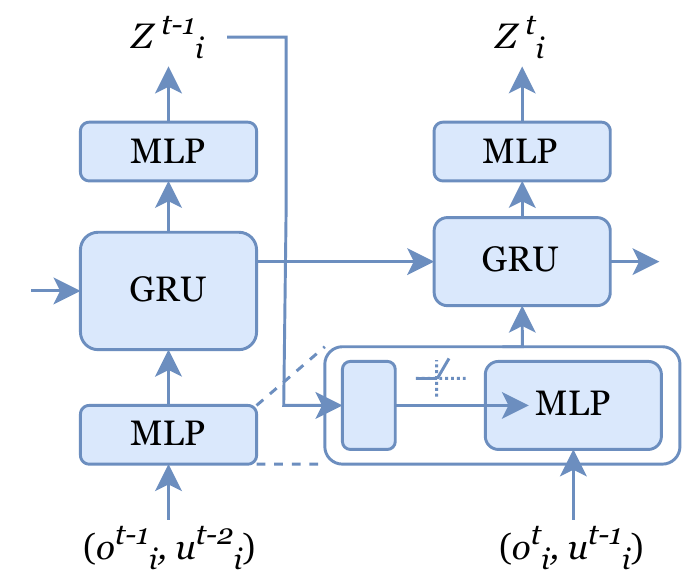}
    \caption{The input layer weights of $\pi_i$ generated using a hyper-network frm previous step's return distribution.}
    \label{fig3}
\end{figure}

Now we focus on the architecture of the local policy, $\pi_i$. As illustrated in Figure \ref{fig2}, we adopt IQN \citep{dabney2018implicit} for sampling distributions over local actions. To address partial observability in local observations, we use a recurrent neural network, as recommended by \cite{hausknecht2015deep}.

The input layer of the local policy is designed using a hyper-network layer, as shown in Figure \ref{fig3}. Since $\pi_i$ predicts a distribution over returns, representing future rewards, we leverage it to dynamically predict weights for the input layer. The hyper-network thus allows for more adaptable weight generation based on the distributional information from future states. A ReLU activation is used to filter the non-negative weights motivated by two conditions: 
\begin{inparaitem}
    \item negative influences on individual observation elements can propagate through subsequent layers, potentially distorting the representation of other elements,
    \item given the inherent uncertainty in the return, an optimistic approach with non-negative weights is preferable.
\end{inparaitem}
For clarity, Figure \ref{fig3} omits the details of IQN, focusing instead on the structure of the hyper-network.

%%%%%%%%%%%%%%%%%%%%%%%%%%%%%%%%%%%%%%%%%%%%%%%%%%%%%%%%%%%%%%%%%%%%%%%%

\section{Theoretical Analysis}
\label{sec5}

In this section we study theoretical guarantees on convergence using our approach and provide sample complexity advancements. Since we are not proposing a new methodology for value decomposition here, we therefore assume the convergence of the background value decomposition methodology used for MARL, studied in more detail here \cite{wang2021towards}. Additionally, we assume Policy Gradient as our optimization approach; and for proving convergence, we assume a parameterized policy $\bm{\pi}_m$ (with policy parameters $m$) in tabular setting with direct policy parameterization.

\begin{theorem}
    Consider a tabular setting for our policy with direct policy parameterization, with the number of policy parameters, $m \in \Delta(\mathcal{U})^{|\mathcal{S}|}$. Let the step size for the NPG update be $\alpha = (1 - \gamma)^{1.5}/ \sqrt{|\mathcal{S}| |\mathcal{U}|T}$, where $T$ is the number of policy update epochs. Then, with confidence $(1 - \delta)$, and assuming convergence of the underlying TD update algorithm, we have the following.
    \begin{equation*}
        V_{\bm{\pi}*}(s_0) - \mathbb{E}[V_{\bm{\pi}_{m_T}}(s_0)] \leq \Theta \left( \sqrt{\frac{|\mathcal{S}| |\mathcal{U}|}{(1 - \gamma)^3 T}} \right)
    \end{equation*}
\end{theorem}

The proof for the above theorem is provided in the Supplementary Material, Section \ref{supp_sec3}. In the following sub-section we formally define our safety criterion, and provide a probabilistic confidence bound on the safety violation for the defined safety criterion.

\begin{figure*}[!ht]
    \centering
    \includegraphics[width=.98\linewidth]{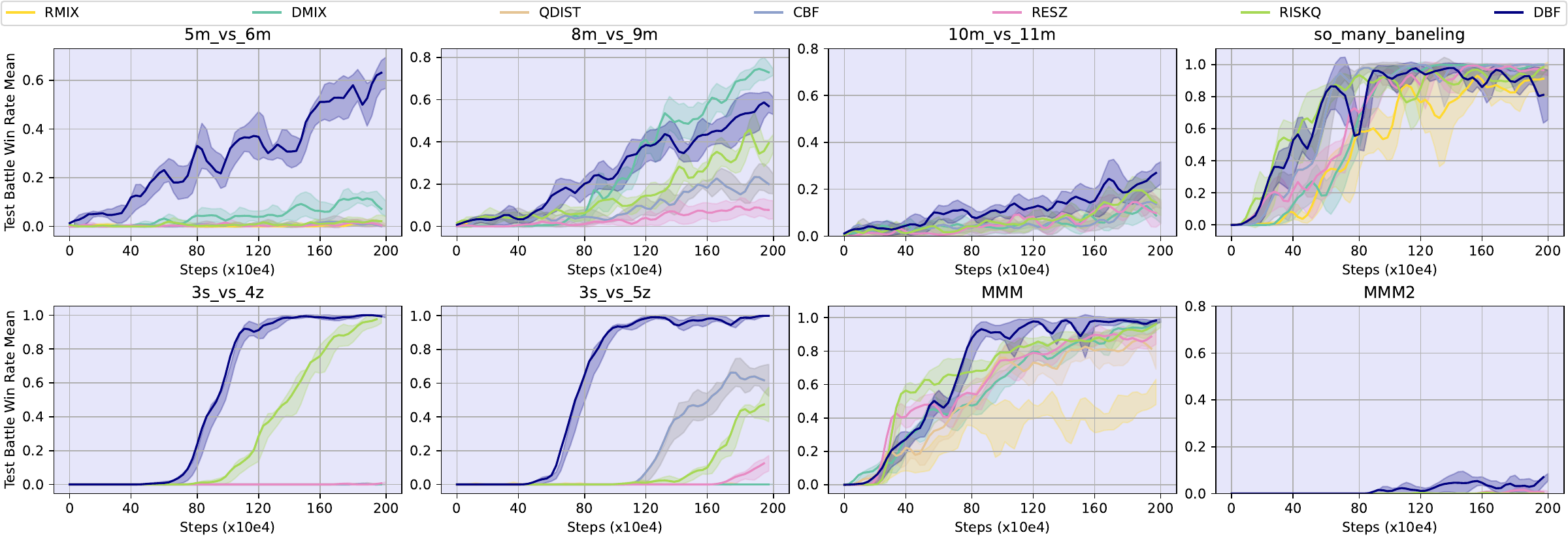}
    \caption{Comparison with the state-of-the-art algorithms using the StarCraft hard and super-hard scenarios. Our method is marked as DBF (Distributional MARL with a Barrier Function based on agent casualties).}
    \label{fig_res1}
\end{figure*}

\subsection{Safety Verification}

With the safety objective, we can verify the safety of the system by introducing a probabilistic verification method for the system. In (MA)RL, classical methods for safety verification the make binary decision about whether the system is safe or unsafe cannot be done, since the system dynamics or the knowledge about the unsafe states are not available. So we address the safety verification from a quantitative safety perspective; i.e., we determine the probability of obtaining an unsafe trajectory remaining within a safety target with some confidence. Here a safe trajectory means reaching the goal state with some agents still alive. 

Our barrier function, although calculated from agent terminations, is not a direct indicator of the terminations. However, we would like to express our constraint in terms of agent terminations. So let $V^{B}_{\bm{\pi}}(\tau)$ be the sum of agent terminations over the trajectory, where $\tau$ is a test trajectory sampled with the trained policy $\bm{\pi}$. Therefore, we can write our objective as a chance-constrained optimization program (CCP) \cite{campi2011sampling} as: given a parameter $\beta \in (0, 1)$, our objective is to compute $\varepsilon \in (0, 1)$ such that with confidence $\geq (1 - \beta)$ our safety constraint is satisfied, i.e., $\tau$ is safe, with probability $\geq (1 - \varepsilon)$. Mathematically, $\tau$ is considered safe if $V^{B}_{\bm{\pi}}(\tau)$ is below a given safety threshold $\omega$. Our objective can be formally restated as, with confidence at least, $(1 - \beta)$,
\begin{equation}
\label{eqn9}
    P(V^{B}_{\bm{\pi}}(\tau) \leq \omega) \geq (1 - \varepsilon)
\end{equation}
Note that $\omega$ can be considered equal to 0; but it may have some disadvantages. When it is evident that there may be casualties in the path to the goal state, taking $\omega = 0$ may harm the return since the policy will deviate from the optimal possible policy owing to the  gradients from the barrier function loss. Thus $\omega$ is kept to be some positive number dependent on the number of agents, $n$; e.g., $\omega = (n-1)$. Finally, from equations \ref{eqn1} and \ref{eqn9}, we can rewrite our constrained optimization problem as follows.
\begin{gather*}
    \arg \max_{\bm{\pi}} \mathbb{E} \left[ \sum_{t \sim \tau} \gamma^{t} r(s_t, \textbf{u}_t) \right] \\
    \text{s.t.} \quad P(V^{B}_{\bm{\pi}}(\tau) \leq \omega) \geq (1 - \varepsilon)
\end{gather*}

We can solve the above CCP problem using a sampling approach in place of the probabilistic constraint, where we substitute the constraint with $N$ i.i.d. sample constraints as: $V^B_{\bm{\pi}}(\tau_i) \leq \omega, \forall i \in [1, N]$\footnote{Please note that here $N$ represents $N$ i.i.d. samples, which is different from $\mathcal{N}$ that represents the set of agents.}. The resulting solution may not be a good approximation of the CCP solution. So for improving the robustness of the solution, we remove some $k$ samples from $N$ \cite{campi2011sampling, singh2024pas}. In other words, we can say that we are removing $k$ out of $N$ constraints from our original problem. The number of removed constraints has a direct impact on the safety probability $(1 - \varepsilon)$. The sample-based optimization of the above problem can therefore be expressed as follows.
\begin{gather*}
    \arg \max_{\bm{\pi}} \mathbb{E} \left[ \sum_{t \sim \tau} \gamma^{t} r(s_t, \textbf{u}_t) \right] \\
    \text{s.t.} \quad V^{B}_{\bm{\pi}}(\tau_i) \leq \omega, \forall i \in \{1, ..., N \} - \mathcal{A}(\{1, ..., N\})
\end{gather*}
where $\mathcal{A}$ is some selection rule. Let $P^N$ be the product of probabilities for $N$ samples. Using the theorem below from \cite{campi2011sampling, singh2024pas}, we can establish the relation between $\varepsilon$ and the confidence parameter, $\beta$, given the values of $N$ and $k$.
% the condition under which $P(V^{B}_{\bm{\pi}}(\tau_i) > \omega) > \varepsilon$ has an arbitrarily small probability $\beta$.

\begin{theorem}[\cite{singh2024pas}]
    Given a small confidence parameter $\beta \in (0, 1)$, if $N$ and $k$ are such that,
    \[
    \binom{k+m-1}{k} \sum_{i=0}^{k+m-1} \binom{N}{i} \varepsilon^i (1 - \varepsilon)^{N-i} \leq \beta
    \]
    where m is the number of policy parameters, then $P^N(P(V^{B}_{\bm{\pi}}(\tau_i) > \omega) \leq \varepsilon) \geq (1 - \beta)$.
\end{theorem}

% The proof is provided in \cite{campi2011sampling}. 
% In terms of safety, the above probability can be rewritten as follows.
% \begin{equation*}
%     P^N \left( P(V^{B}_{\bm{\pi}}(\tau_i) \leq \omega) \leq (1 - \varepsilon) \right) \geq (1 - \beta)
% \end{equation*}
Using the above theorem and from Equation 8 in \cite{campi2011sampling} we know,
\begin{align*}
    & k \leq \varepsilon N - m + 1 - \sqrt{2 \varepsilon N \ln{\frac{(\varepsilon N)^{m-1}}{\beta}}} \\
    \Rightarrow & \frac{(\varepsilon N)^{m-1}}{\beta} \leq e ^{(k+m-1- \varepsilon N)^2 / (2 \varepsilon N)}
\end{align*}
Let us consider ${(k+m-1- \varepsilon N)^2 / (2 \varepsilon N)} = \psi$. Therefore;
\begin{align*}
    % & \frac{(\varepsilon N)^{m-1}}{\beta} \leq e^{\psi} \\
    \varepsilon \leq \left( e^{\ln{\beta} / (m - 1)} \times e^{\psi / (m - 1)}\right) / N
\end{align*}
This is an implicit inequality with $\varepsilon$ occurring on both sides. However, we can deduce that, if we consider $m$ to be finite quantity less than $N$, small value for $\beta$ and $k$, the right hand side of the above equation will be very small. Thus $\varepsilon$, which signifies the upper bound on the probability of unsafe trajectories, will be very small.

%%%%%%%%%%%%%%%%%%%%%%%%%%%%%%%%%%%%%%%%%%%%%%%%%%%%%%%%%%%%%%%%%%%%%%%%

\section{Empirical Studies}
\label{sec6}

In this section we present empirical evaluations of our method. The primary objective is not to introduce a novel constrained MARL approach but to evaluate the impact of incorporating agent casualties as an implicit constraint into the optimization process. The source code is available in \href{https://github.com/somnathhazra/uncertainties_marl}{GitHub}. Our approach is compared with state-of-the-art Distributional MARL algorithms on the StarCraft II~\cite{samvelyan2019starcraft} and MetaDrive~\cite{li2021metadrive} benchmark environments.

\begin{figure*}[!ht]
    \centering
    \includegraphics[width=.98\linewidth]{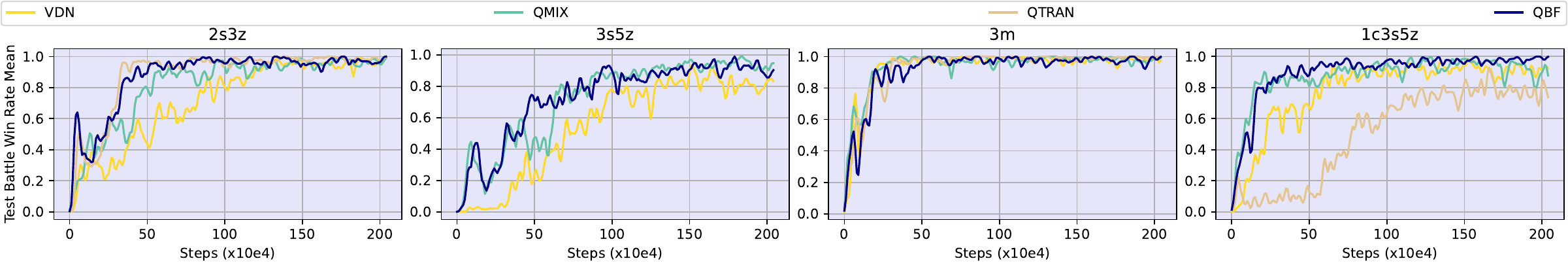}
    \caption{Comparison with the baselines using the StarCraft easy scenarios. Our method is marked as QBF (Q-function evaluation in MARL with a Barrier Function based on agent casualties).}
    \label{fig_res3}
\end{figure*}

\subsection{StarCraft Battle Scenarios}

The StarCarft Multi-Agent Control scenarios include various battle scenarios where allies, controlled by our MARL algorithm, face enemy units controlled by built-in rules. The objective is to maximize the win rate while minimizing ally casualties \cite{samvelyan2019starcraft}. We used the \href{https://github.com/oxwhirl/pymarl}{PyMARL} framework for our experimental setup and training our algorithms. The evaluation metric is the fraction of number of battles won during testing. We derived our barrier function based on the objective of ally casualties and integrate with the reward optimization algorithm to align with the task's objectives. The action space for each of the agent is discrete, thus allowing us to derive distributions over the returns corresponding to each action. We evaluated our approach on two fronts: 
\begin{inparaitem}
    \item firstly we compared with the distributional MARL approaches on hard and super-hard scenarios;
    \item secondly we compared with MARL approaches on easy scenarios.
\end{inparaitem}
For the second case, we did not have distributions over the returns, and could not use them for our policy network as proposed in sub-section \ref{sec4_4}.

\begin{figure}[!h]
    \centering
    \includegraphics[width=0.9\linewidth]{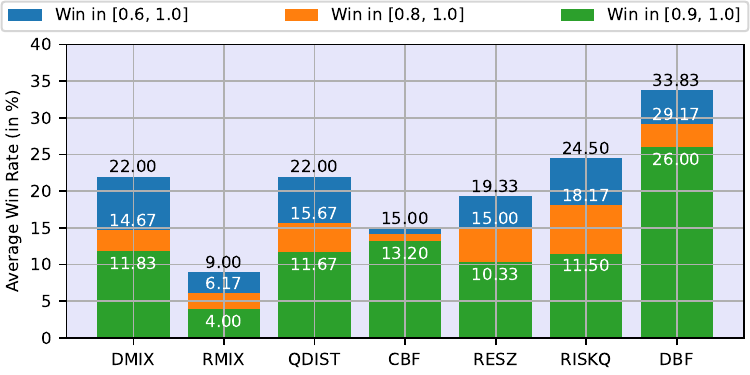}
    \caption{Summary of the evaluations on the StarCraft hard and super-hard scenarios.}
    \label{fig_res2}
\end{figure}

For the first case, we used the following distributional MARL algorithms for comparison:
\begin{inparaitem}
    \item RMIX \cite{qiu2020rmix} uses CVaR with QMIX
    \item DMIX \cite{sun2021dfac} extending QMIX to the distributional setting
    \item QDIST \cite{son2022disentangling} separately analyses the risk measures from the return distributions, based on source of error
    \item CBF \cite{qin2021learning} integrates with DMIX an action correction strategy using a decentralized barrier function approximated using a separate network
    \item RESQ \cite{shen2022resq} uses residual Q-functions
    \item RISKQ \cite{shen2024riskq} uses local risk sensitive action selection policies.
\end{inparaitem}
Our algorithm, referred to as DBF (Distributional MARL with Barrier Function), integrates the barrier loss with DMIX. We used on-policy samples for calculating our barrier function loss \cite{yang2023model}. The results are shown in Figure \ref{fig_res1}. Our algorithm is marked as DBF, which is short for Distributional MARL with a Barrier Function based on agent casualties. Results \ref{fig_res2} indicate DBF achieves a higher fraction of test wins across varying thresholds ($\geq 0.6$, $\geq 0.8$, $\geq 0.9$). From the results, it can be inferred that our approach improves the accuracy of the distributional MARL approaches for the hard and super-hard scenarios. The hard and super-hard scenarios generally contain uneven distribution of agents between the ally and enemy teams; with the enemy team generally containing more number of agents than the ally team. This makes winning battle tougher in these scenarios. We used $\gamma_B = 0.5$ for our experiments. The hyper-parameters used for our experiments are summarized in the Supplementary Material, Section \ref{supp_sec2_3}.

\begin{wrapfigure}{r}{0.6\linewidth}
  \begin{center}
    \includegraphics[width=0.98\linewidth]{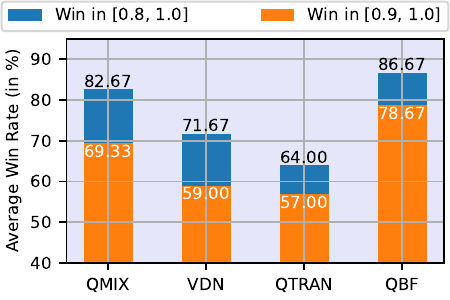}
  \end{center}
  \caption{Summary of evaluations on StarCraft easy scenarios.}
  \label{fig_res4}
\end{wrapfigure}
We also compared our algorithm with traditional MRL approaches:
\begin{inparaitem}
    \item VDN \cite{sunehag2017value}  which uses summation of local returns
    \item QMIX \cite{rashid2020monotonic} uses monotonic factorization of value function
    \item QTRAN \cite{son2019qtran} uses transformers for factorization.
\end{inparaitem}
The traditional MARL approaches have been demonstrated to perform well for the easy scenarios in StarCraft. The easy scenarios contain equal number of ally and enemy agents with similar types of agents for both groups. The action space is discrete here. We integrated our barrier function loss with the QMIX algorithm for evaluation here. The results are shown in figures \ref{fig_res3} and \ref{fig_res4}. Our algorithm is marked as QBF, which stands for Q-function evaluation in MARL with a Barrier Function based on agent casualties. From the evaluations it is evident that integrating the barrier loss improves the performance of the baseline algorithm, albeit marginally. This experiment also highlights the importance of the hyper-network used in the local policy network, explained in Section \ref{sec4_4}.

\subsection{MetaDrive Simulator}

\begin{figure}[!ht]
    \centering
    \includegraphics[width=0.9\linewidth]{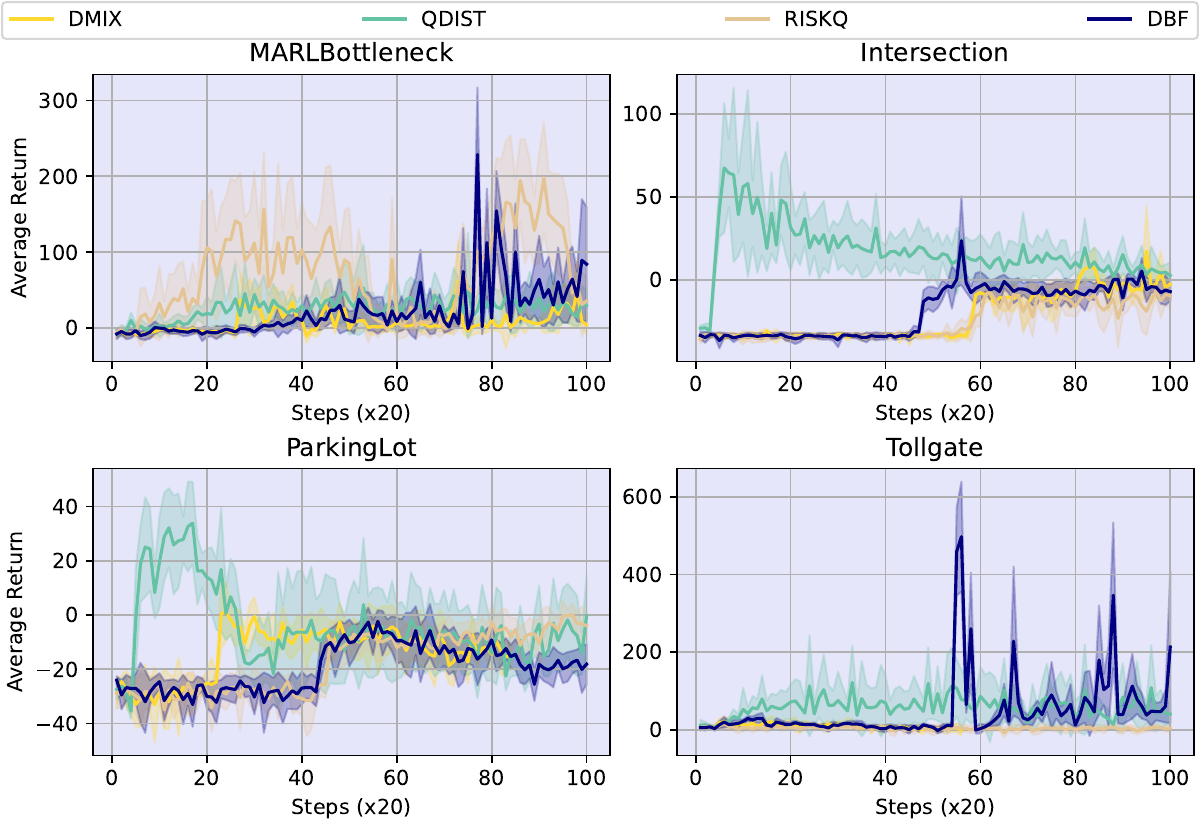}
    \caption{Evaluation results on the Metadrive MARL environments.}
    \label{fig_res7}
\end{figure}

MetaDrive~\cite{li2021metadrive} is a lightweight and realistic driving simulator featuring various multi-agent cooperative driving tasks with decentralized rewards. It presents the challenge of controlling multiple cars in some predefined scenarios. Since the simulator does not return any global state; for training certain algorithms, we used the collective observation vector as the global state. Each agent (vehicle) can execute discrete actions, with crashes or road exits resulting in termination. We did not allow respawn of vehicles for the terminated vehicles. Episodes were terminated when more than half the agents are eliminated during training or when any agent reached termination during evaluation. We used 10 agents for each environment for training.

For our experiments, we used a global reward computed by summing individual agent rewards to align with the POMDP structure. Training spanned 2000 episodes with a replay buffer holding the most recent 1000 episodes. On-policy samples were used for barrier loss calculation. The trained policy was evaluated on 10 episodes at each evaluation step. As shown in Fig.~\ref{fig_res7}, our approach demonstrated strong performance compared to baseline algorithms across scenarios. Key hyper-parameters include $\gamma = 0.99$, $\gamma_B = 0.5$, and the Adam optimizer. Additional experimental details are provided in the Supplementary Material, Section \ref{supp_sec2_1}.

\subsection{Ablation Studies}

\begin{figure}[!ht]
    \centering
    \includegraphics[width=0.9\linewidth]{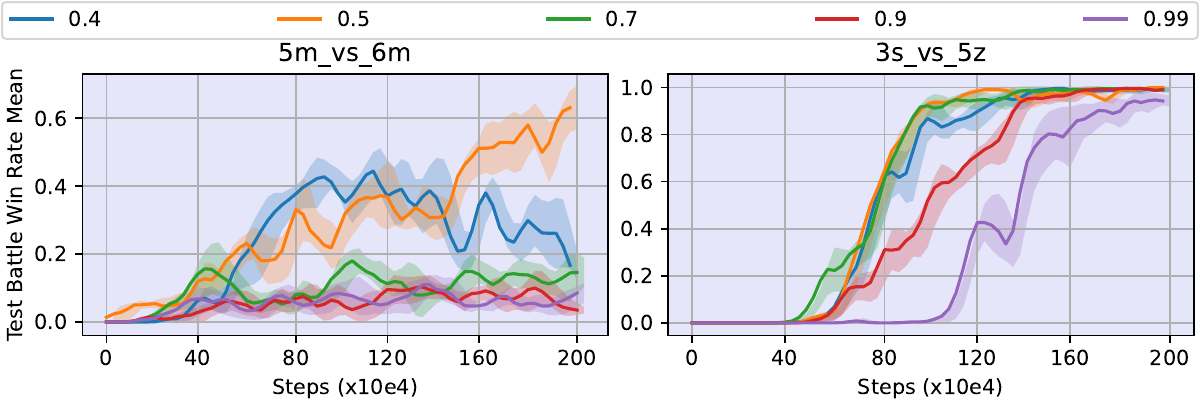}
    \caption{Ablation Study on $\gamma_B$ tested on StarCraft.}
    \label{fig_res5}
\end{figure}

In our experiments, for the discount factor of barrier function, we had used $\gamma_B=0.5$. We had tested various values for the hyper-parameter on two StarCraft scenarios before using them in the rest of our evaluations. The hyper-parameter dictates the creadit assignment horizon length based on the fault metric, i.e., number of agent terminations. From the results in Figure \ref{fig_res5}, it is evident that the algorithm performs well for $\gamma_B=0.5$. The tested values for $\gamma_B$ were 0.4, 0.5, 0.7, 0.9, 0.99.

\begin{figure}
    \centering
    \includegraphics[width=0.95\linewidth]{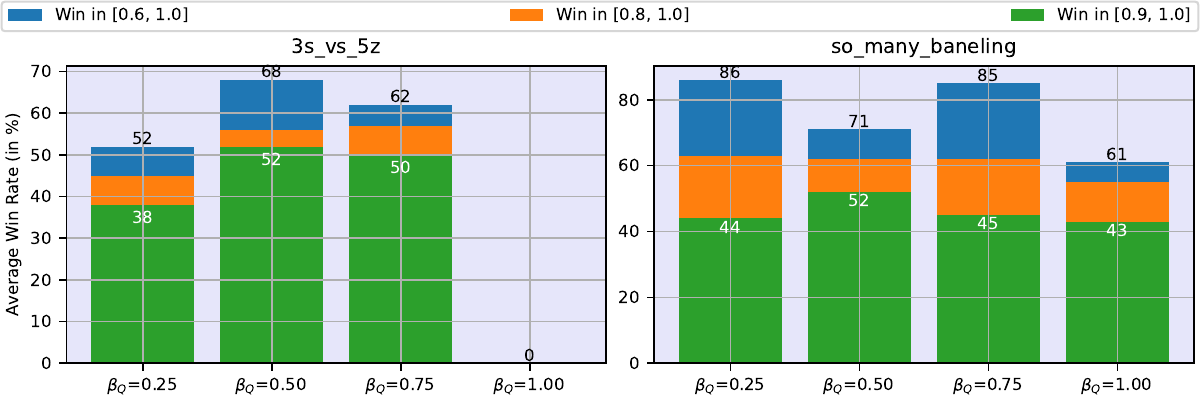}
    \caption{Ablation Study on different values of $\beta_B$ and $\beta_Q$ on StarCraft environments.}
    \label{fig_res6}
\end{figure}

We also conducted an Ablation Study on adjustment of weights of the gradients, $\beta_Q$ and $\beta_B$ on two StarCraft environments. The results are shown in Figure \ref{fig_res6}. Here we only show scores for $\beta_Q$, and $\beta_B = 1 - \beta_Q$. The bars in the figure indicate the percentage of evaluations for which the win-rate was >=0.6, >=0.8, or >=0.9. From the results it is evident that setting equivalent weights for both the gradients, i.e., $\beta_Q = \beta_B = 0.5$, results in better performance for most scenarios, rather than weighing more on any one of the gradients. We used this result for the rest of our evaluations.

%%%%%%%%%%%%%%%%%%%%%%%%%%%%%%%%%%%%%%%%%%%%%%%%%%%%%%%%%%%%%%%%%%%%%%%%

\section{Conclusion}

In this paper, we proposed a novel approach to distributional multi-agent reinforcement learning (MARL) that integrates fault toleration based on an agent safety boundary (based on intrinsic environment dynamics) with effective reward optimization. By incorporating a constraint-based formulation to model agent termination as a crucial learning signal, our method addresses the issue of unsafe actions leading to agent failures, thus improving the stability of MARL in environments where agent survivability is critical. We extended our approach for both traditional MARL value factorization techniques and distributional MARL approaches since the safety constraint is not limited to the distributional analysis. We also suggested an improvement in the local policy network to utilize the predicted distributions over returns.

The integration of these techniques using the PCGrad optimizer resulted in an effective hybrid loss formulation, combining the distributional Huber-quantile loss with our barrier function based safety loss. This allowed us to strike a balance between optimizing agent performance and ensuring system-wide safety. Our experimental results, particularly on complex environments such as StarCraft II micromanagement tasks, and MetaDrive multi-agent driving scenarios, demonstrate that our approach performs comparatively well as compared with the state-of-the-art MARL methods in terms of both convergence and policy robustness. We also theoretically analysed the convergence and safety verification of our approach. The approach can be further be integrated with other constrained MARL approaches to test its robustness in environments with extrinsic safety constraints.
 
%%%%%%%%%%%%%%%%%%%%%%%%%%%%%%%%%%%%%%%%%%%%%%%%%%%%%%%%%%%%%%%%%%%%%%%%

%%% The acknowledgments section is defined using the "acks" environment
%%% (rather than an unnumbered section). The use of this environment 
%%% ensures the proper identification of the section in the article 
%%% metadata as well as the consistent spelling of the heading.

\begin{acks}
The authors acknowledge partial support from AI4CPS IIT Kharagpur grant no TRP3RDTR001 and Ericsson India for this work.
\end{acks}

%%%%%%%%%%%%%%%%%%%%%%%%%%%%%%%%%%%%%%%%%%%%%%%%%%%%%%%%%%%%%%%%%%%%%%%%

%%% The next two lines define, first, the bibliography style to be 
%%% applied, and, second, the bibliography file to be used.

\bibliographystyle{ACM-Reference-Format} 
\bibliography{ref}

%%%%%%%%%%%%%%%%%%%%%%%%%%%%%%%%%%%%%%%%%%%%%%%%%%%%%%%%%%%%%%%%%%%%%%%%

\newpage

\appendix

%%%%%%%%%%%%%%%%%%%%%%%%%%%%%%%%%%%%%%%%%%%%%%%%%%%%%%%%%%%%%%%%%%%%%%%%

\section{Algorithm}

In Algorithm \ref{algo1} below, we outline the steps taken for our policy update. Please note that the loss from the barrier function is calculated using on-policy samples.

%%%%%%%%%%% Algorithm

\begin{algorithm}[!htbp]
    \caption{Policy optimization with barrier function constraint}
    \label{algo1}
    \textbf{Input:} Initial policy $\bm{\pi}$ with weights $m_0$\; \\
    \textbf{Output:} Trained policy weights $m_{\text{out}}$ \; \\
    \begin{algorithmic}[1]
        \FOR{$t=0, ..., T-1$}
            \STATE Policy evaluation under $\bm{\pi_{m_t}}$\;
            \STATE Sample trajectories from replay buffer\;
            \STATE \# Policy update calculation \;
            \STATE Compute TD error from sampled trajectories\;
            \STATE Compute barrier function over the sampled trajectories\;
            \IF{$V^B_{\pi_{m_t}} > \omega$}
                \STATE \# Constraint violation
                \STATE Compute barrier function loss (Equation 5)\;
            \ENDIF
            \STATE \# Gradient correction\;
            \STATE Compute PCGrad from the loss functions\;
            \STATE Update weights: $m_t \rightarrow m_{t+1}$\; \\
            $m_{t+1} = m_t - \alpha \cdot g_1$ ($\theta>90^{\circ}$) = $m_t - \alpha \cdot g_2$ ($\theta \leq 90^{\circ}$)\\
        \ENDFOR
    \end{algorithmic}
    \textbf{Return:} Trained policy $\bm{\pi}_{m_{\text{out}}}$ \;
\end{algorithm}
The output of the algorithm is the final policy. Although, for simplicity, we have used $\bm{\pi}$ to represent the collective policy for all the agents combined, we actually obtain local policies for decentralized evaluation of the agents.

%%%%%%%%%%%%%%%%%%%%%%%%%%%%%%%%%%%%%%%%%%%%%%%%%%%%%%%%%%%%%%%%%%%%%%%%

\section{Empirical Details}
\label{supp_sec2}

\subsection{The StarCraft System}
\label{supp_sec2_1}

The StarCraft system \cite{samvelyan2019starcraft} consists of scenarios that presents the challenge to control multiple cooperative (ally) agents to defeat a group of enemy agents, which are controlled using built-in game AI, in a war simulation. The goal of the ally agents is to win battles against the enemy units across multiple simulation episodes. Each episode ends when all the units of either one among the ally or enemy group dies, or a certain number of timesteps are reached, in which case the battle is considered as a loss for the ally units. There are various types of agents in the system which takes on different roles, and each scenario may include multiple types of units in both the ally and the enemy groups. We explain some components here as an overview, that are of importance to us.

\subsubsection{Observation Space}

SMAC provides a global observation which summarises the environment state including all the participating agents (allies and enemies). It also provides a local observation which also includes the information about all the agents but from the local agent's perspective whose observation is being calculated. Since we use the QMIX algorithm as our backbone, we used both the global and local observations. Each information about other agents in the observation vector includes information about relative position of the agent from the current agent, relative distance of the agent from the current agent, agent type, and some other information. This presents a possibility of analysing importance of each environmental entity for whiich we use the hyper-network. The agents may have different roles to play in the battle, depending on the type of agent. The observation space is continuous, the values are each dimension are real numbers.

\subsubsection{Action Space}

The action space is discrete, where each agent locally derives its action from their local policy $\pi_i$. The actions that can be taken by each agent include move actions, attack actions, stop actions, etc. The attack actions should be directed at a certain agent. This calls for analysis on the distribution of return based on the actions taken, since the return from the action can be realistically captured by a distributional analysis.

\subsection{The MetaDrive Environments}

MetaDrive \cite{li2021metadrive} is a driving simulator with accurate physics simulation that consists of multi-agent scenarios where the task is to control multiple cars in various traffic scenarios, such as Parking Lot, Tollgate, etc. For training our algorithms, in this work we generate 10 agents for each environment. We used high-level scene information as our local observation, which encapsulates nearby vehicle information also. The environments do not separately provide any global information; hence we used stacked local observation as global state information. We used discrete action space for enabling distributional analysis. The action space consists of 25 discrete actions, which is combination of steering and throttle inputs. The reward is a combination of a sparse reward for reaching the goal state, and a dense reward for local guidance. We combined the local rewards available to form the global reward in order to obey the POMDP structure. Although there is a cost function available for the scenarios, we calculated the barrier function based on agent terminations as explained before.

\subsection{Hyper-parameters used}
\label{supp_sec2_3}

\begin{table}[!hb]
    \centering
\caption{Hyper-parameters used for our experiments}
\label{tab_supp1}
    \begin{tabular}{cc}
    \hline
         Hyper-parameter& Value\\
    \hline
         action selector& epsilon greedy\\
 epsilon start&1.0\\
 epsilon end&0.05\\
 batch size&8\\
 buffer size&5000\\
 hidden dimension&64\\
 RNN hidden dimension&64\\
 mixer&DMIX\\
 learning rate&0.001\\
 TD $\lambda$&0.6\\
 optimizer&Adam\\
 \hline
    \end{tabular}

\end{table}

\subsection{Computing Hardware}

We conducted our experiments on a 8th generation Intel(R) Core(TM) i7-9700K CPU @ 3.60GHz processor with 66GB System memory. To boost training of the neural architectures, we used a NVIDIA TITAN RTX GPU with 24GB graphic memory.

%%%%%%%%%%%%%%%%%%%%%%%%%%%%%%%%%%%%%%%%%%%%%%%%%%%%%%%%%%%%%%%%%%%%%%%%

\onecolumn
\section{Proof of Theorem 5.1}
\label{supp_sec3}

Before going into the details, we would like to clarify that for simplicity, we use $\bm{\pi}$ to refer to the collective policy, although we use decentralized policies for evaluation. The convergence of the underlying TD learning algorithm is studied in detail in this text \cite{wang2021towards}. Our algorithm bears resemblance with the policy learning algorithm studied by the authors of CRPO \cite{xu2021crpo}. Along with TD error, they consider approximation errors in estimation of Q-value, which is a very realistic assumption for the multi-agent setup. Hence we refer to their analysis for our derivation in this text. We start with the performance difference lemma, that we use for expressing the performance bound over the policy improvement process.

\begin{lemma}[Performance Difference Lemma \cite{kakade2002approximately}]
    For any two policies $\pi$ and $\pi'$ and any start state distribution, $\rho$,
    \begin{equation}
        V_{\pi}(s_0) - V_{\pi'}(s_0) = \frac{1}{1-\gamma} \mathbb{E}_{s \sim d_\rho} \mathbb{E}_{\textbf{u} \sim \pi} \left[ A_{\pi'}(s, \textbf{u}) \right]
    \end{equation}
    where $V_{\pi}(s_0)$ denotes the accumulated reward over the trajectory with $s_0$ sampled from $\rho$, and $d_\rho$ denotes the state-action visitation distribution under policy $\bm{\pi}$.
    \label{supp_lem1}
\end{lemma}

As in CRPO \cite{xu2021crpo}, any of the policy optimization updates can be used for our setup. Since the specifics of the MARL setup are not required here, to simplify, this section can be better understood by looking at our setup from the point of view of a single global tabular policy learning, $\bm{\pi}$. As is generally done in MARL algorithms, we consider direct policy parameterization with parameters $m \in \Delta(\mathcal{U})^{|\mathcal{S}|}$ in a tabular setting \cite{agarwal2021theory}. Let $\alpha$ represent the learning rate for policy gradient approach.

\begin{lemma}[Performance Improvement Bound]
    For the iterates $\bm{\pi}_{t}$ in the tabular setting and for initial state distribution $\rho$, we have,
    \begin{align*}
        V_{t+1}(s_0) - V_{t}(s_0) \geq
        & \frac{1 - \gamma}{\alpha} \mathbb{E}_{s \sim \rho} \left( \frac{\alpha}{1 - \gamma} \bar{V}_t(s) - \frac{\alpha}{1 - \gamma}V_{t}(s) + \frac{\alpha}{1 - \gamma} \sum_{\textbf{u} \in \mathcal{U}}\bm{\pi}_{t}(s) \left| \bar{Q}_t(s, \textbf{u}) - Q_t(s, \textbf{u}) \right| \right)
        \\
        & - \frac{1}{1 - \gamma} \mathbb{E}_{s \sim d_{\rho}} \sum_{\textbf{u} \in \mathcal{U}}\bm{\pi}_{t}(s) \left| \bar{Q}_t(s, \textbf{u}) - Q_t(s, \textbf{u}) \right|
        - \frac{1}{1 - \gamma} \mathbb{E}_{s \sim d_{\rho}} \sum_{\textbf{u} \in \mathcal{U}}\bm{\pi}_{t+1}(s) \left| Q_t(s, \textbf{u}) - \bar{Q}_t(s, \textbf{u}) \right|
    \end{align*}
    where $\bm{\pi}_{m_t} = \bm{\pi}_t$, $V_{\bm{\pi}_{m_t}} = V_t$, $Q_{\bm{\pi}_{m_t}} = Q_t$, $\bar{Q}_t(s, \textbf{u})$ is the approximated Q-function estimated using the policy $\bm{\pi}_{m_t}$, and $\bar{V}_t(s) = \sum_{\textbf{u} \in \mathcal{U}} \bar{Q}_t(s, \textbf{u})$.
    \label{supp_lem2}
\end{lemma}

\begin{proof}
    From Lemma \ref{supp_lem1}, we know that,
    \begin{align*}
        V_{t+1}(s_0) - V_{t}(s_0) & = \frac{1}{1-\gamma} \mathbb{E}_{s \sim d_\rho} \sum_{\textbf{u} \in \mathcal{U}} \bm{\pi}_{t+1} A_{t}(s, \textbf{u}) 
        \\
        & = \frac{1}{1-\gamma} \mathbb{E}_{s \sim d_\rho} \sum_{\textbf{u} \in \mathcal{U}} \bm{\pi}_{t+1} Q_{t}(s, \textbf{u}) - \frac{1}{1-\gamma} \mathbb{E}_{s \sim d_\rho} V_t(s)
        \\
        & = \frac{1}{1-\gamma} \mathbb{E}_{s \sim d_\rho} \sum_{\textbf{u} \in \mathcal{U}} \bm{\pi}_{t+1} \bar{Q}_{t}(s, \textbf{u}) + \frac{1}{1-\gamma} \mathbb{E}_{s \sim d_\rho} \sum_{\textbf{u} \in \mathcal{U}} \bm{\pi}_{t+1} \left( Q_{t}(s, \textbf{u}) - \bar{Q}_{t}(s, \textbf{u}) \right) - \frac{1}{1-\gamma} \mathbb{E}_{s \sim d_\rho} V_t(s)
        \\
        & = \frac{1}{1-\gamma} \mathbb{E}_{s \sim d_\rho} \bar{V}_{t}(s, \textbf{u}) + \frac{1}{1-\gamma} \mathbb{E}_{s \sim d_\rho} \sum_{\textbf{u} \in \mathcal{U}} \bm{\pi}_{t+1} \left( Q_{t}(s, \textbf{u}) - \bar{Q}_{t}(s, \textbf{u}) \right) - \frac{1}{1-\gamma} \mathbb{E}_{s \sim d_\rho} V_t(s)
        \\
        & = \frac{1}{1-\gamma} \mathbb{E}_{s \sim d_\rho} \bar{V}_{t}(s, \textbf{u}) + \frac{1}{1-\gamma} \mathbb{E}_{s \sim d_\rho} \sum_{\textbf{u} \in \mathcal{U}} \bm{\pi}_{t+1} \left( Q_{t}(s, \textbf{u}) - \bar{Q}_{t}(s, \textbf{u}) \right) - \frac{1}{1-\gamma} \mathbb{E}_{s \sim d_\rho} V_t(s)
        \\
        & = \frac{1}{1-\gamma} \mathbb{E}_{s \sim d_\rho} \bar{V}_{t}(s, \textbf{u}) + \frac{1}{1-\gamma} \mathbb{E}_{s \sim d_\rho} \sum_{\textbf{u} \in \mathcal{U}} \bm{\pi}_{t+1} \left( Q_{t}(s, \textbf{u}) - \bar{Q}_{t}(s, \textbf{u}) \right) - \frac{1}{1-\gamma} \mathbb{E}_{s \sim d_\rho} V_t(s) 
        \\
        & \quad + \frac{1}{1-\gamma} \mathbb{E}_{s \sim d_\rho} \sum_{\textbf{u} \in \mathcal{U}} \bm{\pi}_{t} \left| \bar{Q}_{t}(s, \textbf{u}) - Q_{t}(s, \textbf{u}) \right|  - \frac{1}{1-\gamma} \mathbb{E}_{s \sim d_\rho} \sum_{\textbf{u} \in \mathcal{U}} \bm{\pi}_{t} \left| \bar{Q}_{t}(s, \textbf{u}) - Q_{t}(s, \textbf{u}) \right|
        \\
         & \geq \frac{1}{\alpha} \mathbb{E}_{s \sim d_\rho} \left( \frac{\alpha}{1-\gamma} \bar{V}_{t}(s, \textbf{u}) - \frac{\alpha}{1-\gamma} V_t(s) + \frac{\alpha}{1-\gamma} \sum_{\textbf{u} \in \mathcal{U}} \bm{\pi}_{t} \left| \bar{Q}_{t}(s, \textbf{u}) - Q_{t}(s, \textbf{u}) \right| \right)
         \\
         & \quad - \frac{1}{1-\gamma} \mathbb{E}_{s \sim d_\rho} \sum_{\textbf{u} \in \mathcal{U}} \bm{\pi}_{t} \left| \bar{Q}_{t}(s, \textbf{u}) - Q_{t}(s, \textbf{u}) \right|  - \frac{1}{1-\gamma} \mathbb{E}_{s \sim d_\rho} \sum_{\textbf{u} \in \mathcal{U}} \bm{\pi}_{t+1} \left| Q_{t}(s, \textbf{u}) - \bar{Q}_{t}(s, \textbf{u}) \right|
         \\
         & \geq \frac{1 - \gamma}{\alpha} \mathbb{E}_{s \sim d_\rho} \left( \frac{\alpha}{1-\gamma} \bar{V}_{t}(s, \textbf{u}) - \frac{\alpha}{1-\gamma} V_t(s) + \frac{\alpha}{1-\gamma} \sum_{\textbf{u} \in \mathcal{U}} \bm{\pi}_{t} \left| \bar{Q}_{t}(s, \textbf{u}) - Q_{t}(s, \textbf{u}) \right| \right)
         \\
         & \quad - \frac{1}{1-\gamma} \mathbb{E}_{s \sim d_\rho} \sum_{\textbf{u} \in \mathcal{U}} \bm{\pi}_{t} \left| \bar{Q}_{t}(s, \textbf{u}) - Q_{t}(s, \textbf{u}) \right|  - \frac{1}{1-\gamma} \mathbb{E}_{s \sim d_\rho} \sum_{\textbf{u} \in \mathcal{U}} \bm{\pi}_{t+1} \left| Q_{t}(s, \textbf{u}) - \bar{Q}_{t}(s, \textbf{u}) \right|
    \end{align*}
    This concludes our proof.
\end{proof}

It is to be noted that since we are using a gradient manipulation step, the improvement bound at each step is multiplied by the constant used for weighing the gradients. However, since it is a constant, and here we are calculating the time-complexity for simplicity we have removed the weight from our calculation. Also, since $\bm{\pi}$ is not only a function of $s$, we refrain from using $\bm{\pi}(s)$ throughout the text.

\begin{lemma}[Upper Bound on the Optimality Gap]
     Consider the policy gradient updates in the tabular setting, we have,
     \begin{align*}
         V_*(s_0) - V_t(s_0) \leq \frac{2 \alpha v^2_{\texttt{max}} |\mathcal{S}| |\mathcal{U}|}{(1 - \gamma)^3} + \frac{3(1 + \alpha v_{\texttt{max}})}{(1 - \gamma)^2} \left\lVert Q_t - \bar{Q}_t \right\rVert_2 
     \end{align*}
     where $V_*(s_0) = V_{\bm{\pi}^*}(s_0)$, and the reward $r \in [0, v_{\texttt{max}}]$. 
     \label{supp_lem3}
\end{lemma}

\begin{proof}
    From Lemma \ref{supp_lem1}, we can say that,
    \begin{align*}
        V_*(s_0) - V_{t}(s_0) & = \frac{1}{1-\gamma} \mathbb{E}_{s \sim d_*} \sum_{\textbf{u} \in \mathcal{U}} \bm{\pi}^* A_{t}(s, \textbf{u})
        \\
        & = \frac{1}{1-\gamma} \mathbb{E}_{s \sim d_*} \sum_{\textbf{u} \in \mathcal{U}} \bm{\pi}^* Q_{t}(s, \textbf{u}) - \frac{1}{1-\gamma} \mathbb{E}_{s \sim d_*} V_t(s)
        \\
        & = \frac{1}{1-\gamma} \mathbb{E}_{s \sim d_*} \sum_{\textbf{u} \in \mathcal{U}} \bm{\pi}^* \bar{Q}_{t}(s, \textbf{u}) + \frac{1}{1-\gamma} \mathbb{E}_{s \sim d_*} \sum_{\textbf{u} \in \mathcal{U}} \bm{\pi}^* \left( Q_{t}(s, \textbf{u}) - \bar{Q}_{t}(s, \textbf{u}) \right) - \frac{1}{1-\gamma} \mathbb{E}_{s \sim d_*} V_t(s)
        \\
        & \leq \frac{1}{\alpha} \mathbb{E}_{s \sim d_*} \left( \frac{\alpha}{1-\gamma} \bar{V}_{t}(s, \textbf{u}) - \frac{\alpha}{1-\gamma} V_t(s) + \frac{\alpha}{1-\gamma} \sum_{\textbf{u} \in \mathcal{U}} \bm{\pi}_{t} \left| \bar{Q}_{t}(s, \textbf{u}) - Q_{t}(s, \textbf{u}) \right| \right)
        \\
        & \quad + \frac{1}{1-\gamma} \mathbb{E}_{s \sim d_*} \sum_{\textbf{u} \in \mathcal{U}} \bm{\pi}^* \left( Q_{t}(s, \textbf{u}) - \bar{Q}_{t}(s, \textbf{u}) \right)
        \\
        & \leq \frac{1}{1 - \gamma} \left( V_{t+1}(s_0) - V_{t}(s_0) \right)
        \\
        & \quad + \frac{1}{(1-\gamma)^2} \mathbb{E}_{s \sim d_*} \sum_{\textbf{u} \in \mathcal{U}} \bm{\pi}_{t+1} \left| Q_{t}(s, \textbf{u}) - \bar{Q}_{t}(s, \textbf{u}) \right| + \frac{1}{(1-\gamma)^2} \mathbb{E}_{s \sim d_*} \sum_{\textbf{u} \in \mathcal{U}} \bm{\pi}_{t} \left| Q_{t}(s, \textbf{u}) - \bar{Q}_{t}(s, \textbf{u}) \right|
        \\
        & \quad + \frac{1}{1-\gamma} \mathbb{E}_{s \sim d_*} \sum_{\textbf{u} \in \mathcal{U}} \bm{\pi}^* \left( Q_{t}(s, \textbf{u}) - \bar{Q}_{t}(s, \textbf{u}) \right) & \text{[from Lemma \ref{supp_lem2}]}
        \\
        & \leq \frac{2 v_{\texttt{max}}}{(1 - \gamma)^2} \left\lVert m_{t+1} - m_t \right\rVert_2 + \frac{3}{(1 - \gamma)^2} \left\lVert Q_t - \bar{Q}_t \right\rVert_2 
        \quad = \frac{2 \alpha v_{\texttt{max}}}{(1 - \gamma)^2} \left\lVert \bar{Q}_t \right\rVert_2 + \frac{3}{(1 - \gamma)^2} \left\lVert Q_t - \bar{Q}_t \right\rVert_2 & \text{[from \cite{xu2020improving}]}
        \\
        & \leq \frac{2 \alpha v_{\texttt{max}}}{(1 - \gamma)^3} \left\lVert Q_t \right\rVert_2 + \frac{3(1 + \alpha v_{\texttt{max}})}{(1 - \gamma)^2} \left\lVert Q_t - \bar{Q}_t \right\rVert_2
        \quad \leq \frac{2 \alpha v^2_{\texttt{max}} |\mathcal{S}| |\mathcal{U}|}{(1 - \gamma)^3} + \frac{3(1 + \alpha v_{\texttt{max}})}{(1 - \gamma)^2} \left\lVert Q_t - \bar{Q}_t \right\rVert_2
    \end{align*}
    This concludes our proof.
\end{proof}

The lemmas \ref{supp_lem2} and \ref{supp_lem3} proved above also hold true the barrier function that we had used for applying the loss due to the constraint. Owing to better convergence (proved by the authors in \cite{qin2021learning}) we have used the loss function based on the barrier function. For time complexity analysis, we assume a parallel convergence analysis for the constraint holds true, with our constraints scaling with the rewards to obey the limit. In general, TD error can also be used to learn the constraints here.

Let $T$ be the total number of training steps. Let us assume, with probability at least $1 - \delta/T$ we have, $\left\lVert Q_t - \bar{Q}_t \right\rVert_2 \leq \sqrt{(1 - \gamma) |\mathcal{S}| |\mathcal{U}|} / \sqrt{T}$. Given that, with probability $1 - \delta$ we have $\sum_T \left\lVert Q_t - \bar{Q}_t \right\rVert_2 \leq \sqrt{(1 - \gamma) |\mathcal{S}| |\mathcal{U}| T}$, we have the following with same confidence.

\begin{equation*}
    \alpha \sum_T \left( (V_*(s_0) - V_{t}(s_0)) + (V^B_*(s_0) - V^B_{t}(s_0)) \right) \leq \frac{2 \alpha^2 v^2_{\texttt{max}} |\mathcal{S}| |\mathcal{U}| T}{(1 - \gamma)^3} + \frac{3 \alpha (1 + \alpha v_{\texttt{max}}) \sqrt{|\mathcal{S}| |\mathcal{U}| T}}{(1 - \gamma)^{1.5}}
\end{equation*}

Let $\alpha = (1 - \gamma)^{1.5} / \sqrt{|\mathcal{S}| |\mathcal{U}| T}$. If we consider gradient update due to the reward and gradient update due to the constraint to be separate steps, from Lemma 9 in \cite{xu2021crpo} we know that reward optimization runs for more than half of the time-steps. Therefore,

\begin{equation*}
    V_*(s_0) - V_{T}(s_0) \leq \frac{2}{\alpha T} \left( \frac{2 \alpha^2 v^2_{\texttt{max}} |\mathcal{S}| |\mathcal{U}| T}{(1 - \gamma)^3} + \frac{3 \alpha (1 + \alpha v_{\texttt{max}}) \sqrt{|\mathcal{S}| |\mathcal{U}| T}}{(1 - \gamma)^{1.5}} \right)
    \leq \frac{|\mathcal{S}| |\mathcal{U}|}{(1 - \gamma)^{1.5} \sqrt{T}} (4 v^2_{\texttt{max}} + 6 + 6 v_{\texttt{max}})
\end{equation*}

Using the same equation, we can obtain the bound on the constraints; which also has the same bound.

%%%%%%%%%%%%%%%%%%%%%%%%%%%%%%%%%%%%%%%%%%%%%%%%%%%%%%%%%%%%%%%%%%%%%%%%

% \twocolumn

\begin{figure*}[!ht]
    \centering
    \includegraphics[width=0.9\linewidth]{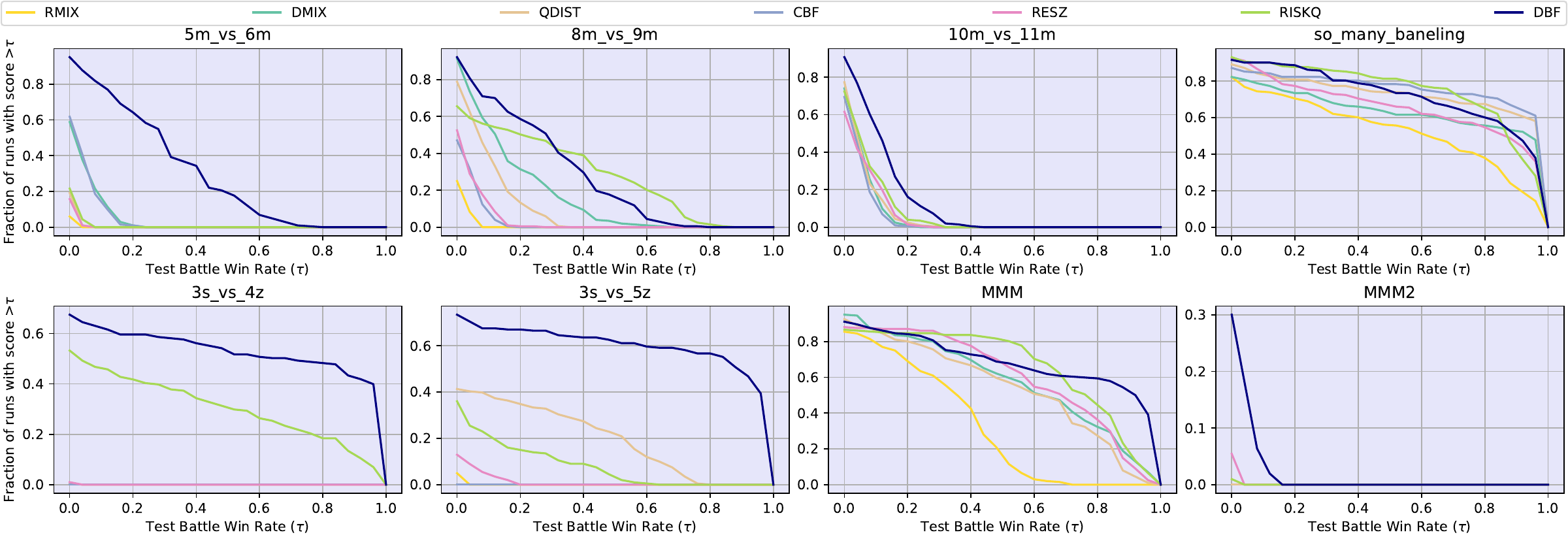}
    \caption{Analysis of win rate on the StarCraft hard and super-hard scenarios.}
    \label{fig_supp1}
\end{figure*}

\begin{figure*}[!ht]
    \centering
    \includegraphics[width=0.9\linewidth]{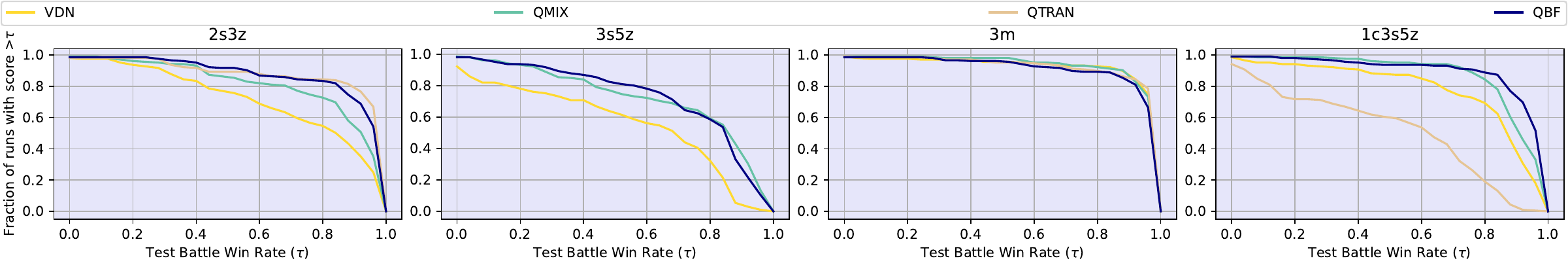}
    \caption{Analysis of win rate on the StarCraft easy scenarios.}
    \label{fig_supp2}
\end{figure*}

\section{Additional Experimental Results}

We performed additional analysis on the evaluations done using the StartCraft environments. Here the fraction of trajectories with test battle win rate above a certain fraction as indicated in the horizontal axis, highlighting its convergence. The results are shown in figures \ref{fig_supp1}, \ref{fig_supp2}.

\end{document}